\documentclass{llncs}
\usepackage{amsmath}
\usepackage{times}
\usepackage{indentfirst}
\usepackage{graphicx}
\usepackage{subfigure}
\usepackage{amssymb}
\makeatletter
\newif\if@borderstar
\def\bordermatrix{\@ifnextchar*{%
\@borderstartrue\@bordermatrix@i}{\@borderstarfalse\@bordermatrix@i*}%
}
\def\@bordermatrix@i*{\@ifnextchar[{\@bordermatrix@ii}{\@bordermatrix@ii[()]}}
\def\@bordermatrix@ii[#1]#2{%
\begingroup
\m@th\@tempdima8.75\p@\setbox\z@\vbox{%
\def\cr{\crcr\noalign{\kern 2\p@\global\let\cr\endline }}%
\ialign {$##$\hfil\kern 2\p@\kern\@tempdima &\thinspace %
\hfil $##$\hfil &&\quad\hfil $##$\hfil\crcr\omit\strut %
\hfil\crcr\noalign{\kern -\baselineskip}#2\crcr\omit %
\strut\cr}}%
\setbox\tw@\vbox{\unvcopy\z@\global\setbox\@ne\lastbox}%
\setbox\tw@\hbox{\unhbox\@ne\unskip\global\setbox\@ne\lastbox}%
\setbox\tw@\hbox{%
$\kern\wd\@ne\kern -\@tempdima\left\@firstoftwo#1%
\if@borderstar\kern2pt\else\kern -\wd\@ne\fi%
\global\setbox\@ne\vbox{\box\@ne\if@borderstar\else\kern 2\p@\fi}%
\vcenter{\if@borderstar\else\kern -\ht\@ne\fi%
\unvbox\z@\kern-\if@borderstar2\fi\baselineskip}%
\if@borderstar\kern-2\@tempdima\kern2\p@\else\,\fi\right\@secondoftwo#1$%
}\null \;\vbox{\kern\ht\@ne\box\tw@}%
\endgroup
}
\makeatother
\begin{document}

\title{Matrix approach to rough sets through vector matroids over a field}         

\author{Aiping Huang, William Zhu~\thanks{Corresponding author.
E-mail: williamfengzhu@gmail.com (William Zhu)}}
\institute{
Lab of Granular Computing,\\
Zhangzhou Normal University, Zhangzhou 363000, China}



\date{\today}          
\maketitle

\begin{abstract}
Rough sets were proposed to deal with the vagueness and incompleteness of knowledge in information systems.
There are many optimization issues in this field such as attribute reduction.
Matroids generalized from matrices are widely used in optimization.
Therefore, it is necessary to connect matroids with rough sets.
In this paper, we take field into consideration and introduce matrix to study rough sets through vector matroids.
First, a matrix representation of an equivalence relation is proposed, and then a matroidal structure of rough sets over a field is presented by the matrix.
Second, the properties of the matroidal structure including circuits, bases and so on are studied through two special matrix solution spaces, especially null space.
Third, over a binary field, we construct an equivalence relation from matrix null space, and establish an algebra isomorphism from the collection of equivalence relations to the collection of sets, which any member is a family of the minimal non-empty sets that are supports of members of null space of a binary dependence matrix.
In a word, matrix provides a new viewpoint to study rough sets.

\end{abstract}

\clearpage
\section{Introduction}
The vagueness and incompleteness of knowledge are commom phenomena in information systems.
Rough set theory~\cite{Pawlak82Rough}, based on equivalence relations (resp. partitions), was proposed by Pawlak in hybrid approaches to
improve the performance of data analysis tools.
This technique has led to many practical applications in various areas, such as attribute reduction~\cite{HeWuChenZhao11Fuzzy,FanZhu12Attribute,MinHeQianZhu11Test,YaoZhao08Attribute}, feature selection~\cite{DashLiu03Consistency-based,HuYuLiuWu08Neighborhood,TsengHuang07Rough},
rule extraction~\cite{Baesens2003,Cruz-Cano2012,DuHuZhuMa11Rule,WangTsangZhaoChenYeung07Learning}, and so on.
In order to generalize the rough set theory's applications, some scholars have extended rough sets to generalized rough sets based on tolerance relation~\cite{SkowronStepaniuk96tolerance}, similarity relation~\cite{SlowinskiVanderpooten00AGeneralized} and arbitrary binary relation~\cite{LiuZhu08TheAlgebraic,ZhuWang06ANew,Zhu09RelationshipBetween}.
Through extending a partition to a covering, rough sets have been extended to covering-based rough sets~\cite{WangZhu12Quantitative,ZhuWang07OnThree,Zhu07Topological,ZhuWang03Reduction}.
Matroid theory also has been promoted further to study rough set theory and its applications~\cite{Edmonds71Matroids}.

Matroid theory~\cite{Lai01Matroid,Oxley93Matroid} borrows extensively from linear algebra theory and graph theory.
There are dozens of equivalent ways to define a matroid.
Significant definitions of a matroid include those in terms of independent sets, bases, circuits, closed sets (resp. flats)
and rank functions, which provides well-established platforms to connect with other
theories. In applications, matroids have been widely used in many fields such as combinatorial
optimization, network flows and algorithm design, especially greedy algorithm
design~\cite{Edmonds71Matroids,Lawler01Combinatorialoptimization}.
In recent years, there are many fruitful achievements about the connection between matroids and rough sets~\cite{BarnabeiNicolettiPezzoli98Matroids,Huang2012,LiLiu12Matroidal,WangZhu11Matroidal,WangZhuZhuMin12Matroidal,WangZhu12Quantitative,ZhuWang11Matroidal}.

Matrix, which is a good computational tool and easy to represent, compute and accelerate, finds many applications in most scientific fields.
In physics, it is used to study physical phenomena, such as the motion of rigid bodies.
In computer graphics, it is used to project a 3-dimensional image onto a 2-dimensional screen.
In probability theory and statistics, stochastic matrices are used to describe sets of probabilities; for instance, they are used within the PageRank algorithm that ranks the pages in a Google search.
As an approach to study rough sets, matrix has existed in many papers~\cite{HuangZhu12On,Liu10Closures,Liu06TheTransitive,Liu10Rough}.

In this paper, through another way, namely vector matroids, we introduce matrix to study rough sets.
First, an approach to construct a matroid is introduced from the viewpoint of set theory, and a matrix representation of an equivalence relation was proposed.
Over a field, though proving the matroid is the same as the one induced by the matrix through vector matroids, we construct a matroidal structure of rough sets over a field from matrix.
Second, we introduce two special matrix solution spaces, especially null space, to study the characteristics of the matroid.
Over a field, one matrix with entries can induce a vector matroid.
It is interesting that the circuits of the matroid are the minimal non-empty sets that are supports of members of the null space of the matrix over the field.
Third, we take binary field into consideration and construct an equivalence relation from a matrix null space.
Moreover, we find that a collection of equivalence relations and a collection of sets, which any member is a collection of the minimal non-empty sets that are supports of members of null space of a binary dependence matrix, are algebra isomorphism.
In a word, this work indicates that we can study rough sets from the viewpoint of matrix.

The rest of this paper is arranged as follows.
Section \ref{S:Basicdefinitions} reviews some fundamental concepts related to rough set theory, matroid theory and linear algebra theory.
In section \ref{S:Representationofmatroidinducedbyanequivalencerelationoverafiled}, we study a matroidal structure of rough sets over a field through matrix.
Section \ref{S:Nullspaceapproachtoroughsetthroughmatroid} introduces two special matrix solution spaces, especially null space, to study the matroidal structure.
In section \ref{Equivalencerelationinducedbymatrix}, we construct an equivalence relation from the matrix null space over binary field and obtain an algebra isomorphism system between a collection of equivalence relations and a collection of sets which any member is a family of the minimal non-empty sets that are supports of members of null space of a binary dependence matrix.
Section \ref{S:Conclusions} concludes this paper.

\section{Basic definitions}
\label{S:Basicdefinitions}

In this section, we present some fundamental concepts about rough set theory, matroid theory and linear algebra theory.
First of all, we review some basic concepts of rough set theory.

\subsection{Rough set theory}
In Pawlak's rough set theory, the lower and upper approximation operations are two key concepts.
An equivalence relation, that is, a partition, is the simplest formulation of the lower and upper approximation operations.

Let $U$ be a finite set and $R$ an equivalence relation on $U$.
$R$ will generate a partition $U/R=\{P_{1}, P_{2}, \cdots, P_{s}\}$ on $U$, where $P_{1}, P_{2}, \cdots, P_{s}$ are the equivalence classes
generated by $R$.
$\forall X\subseteq U$, the lower and upper approximations of $X$ are defined as follows, respectively.\\
 $R_{\ast}(X)=\bigcup\{P_{i}\in U/R:P_{i}\subseteq X\}$,\\
$R^{\ast}(X)=\bigcup\{P_{i}\in U/R:P_{i}\bigcap X\neq \emptyset\}$.

\subsection{Linear algebra theory}

In this subsection, we introduce some basic concepts of linear algebra theory used in this paper.
Field plays an important role in linear algebra theory, we introduce the concept firstly.

\begin{definition}(Field)\cite{Xiong84Modern}
\label{field}
A field is defined as a set together with two operations, usually called addition and multiplication, and denoted by $+$ and $\cdot$,
respectively, such that the following axioms hold (subtraction and division are defined implicitly in terms of the inverse operations of addition and multiplication):\\
(1): For all $a, b \in F$, $a+b \in F$ and $a \cdot b \in F$.\\
(2): For all $a, b, c \in F$, $a + (b + c) = (a + b) + c$ and $(a \cdot b) \cdot c = a \cdot (b \cdot c )$.\\
(3): For all $a, b \in F$, $a + b = b + a$ and $a \cdot b =b \cdot a$.\\
(4): There exists an element of $F$, called the additive identity element and denoted by $0$, such that for all $a \in F$, $a + 0 = a$.
Likewise, there is an element, called the multiplicative identity element and denoted by $1$, such that for all $a \in F$, $a \cdot 1 = a$.\\
(5): For every $a \in F$, there exists an element $-a \in F$ such that $a + (-a) = 0$.
Similarly, for any $a \in F$ other than $0$, there exists an element $a^{-1} \in F$ such that $a \cdot a^{-1} = 1$.\\
(6) For all $a, b, c \in F$, the following equalities holds: $ a \cdot (b + c) = (a \cdot b) + (a \cdot c)$ and $(b + c) \cdot a = (b \cdot a) + (c \cdot a)$.
\end{definition}

A field is therefore an algebraic structure $<F, +, \cdot, -, ^{-1}, 0, 1>$.
Generally, for a field $F$ and positive integer $n$, $V(n, F)$ denotes the $n-$dimensional vector space over $F$.
Any element of $V(n, F)$ is denoted as $\mathbf{v}=(v_{1},v_{2}, \cdots, v_{n})^{T}$ where $v_{i} \in F$ ($ 1\leq i \leq n$).
The operations on $V(n, F)$ are established as follows.
For all $\mathbf{v}=(v_{1}, v_{2}, \cdots, v_{n})^{T} \in V(n, F)$,
$\mathbf{v}^{'}=(v_{1}^{'}, v_{2}^{'}, \cdots, v_{n}^{'})^{T} \in V(n, F)$ and $k \in F$,
$k \mathbf{v}=(kv_{1}, kv_{2}, \cdots, kv_{n})^{T}$ and $\mathbf{v}+\mathbf{v}^{'}=(v_{1}+v_{1}^{'}, v_{2}+v_{2}^{'}, \cdots, v_{n}+v_{n}^{'})^{T}$.

Null space, as an important concept in linear algebra theory, provides us a method to study rough sets in this paper.

\begin{definition}\cite{Lay2010}
Let $F$ be a field and $A$ an $m \times n$ matrix over $F$.
The null space of an $m \times n$ matrix $A$, written as $\mathcal{N}_{F}(A)$, is the set of all
solutions to the homogeneous equation $A\mathbf{x}=\mathbf{0}$.
In set notion, $\mathcal{N}_{F}(A)=\{\mathbf{x}\in V(n, F): A \mathbf{x}=\mathbf{0}\}$.
\end{definition}

Vectors $\mathbf{v}_{1}, \mathbf{v}_{2}, \cdots, \mathbf{v}_{n} \in  V(n, F)$ are said to be linearly independent over $F$ if there exist $x_{1}, x_{2}, \cdots, x_{n} \in F$ such that the vector equation $x_{1}\mathbf{v}_{1} + x_{2}\mathbf{v}_{2} + \cdots +x_{n}\mathbf{v}_{n}=\mathbf{0}$ has only the trivial solution,
and are said to be linearly dependent over $F$ if there exist $c_{1}, c_{2}, \cdots, c_{p} \in F$, not all zero, such that $c_{1}\mathbf{v}_{1} + c_{2}\mathbf{v}_{2} + \cdots +c_{n}\mathbf{v}_{n}= \mathbf{0}$.
The rank of matrix $A$ over $F$ is the maximum number of linearly independent columns in $A$ and the maximum number of linearly independent columns in $A^{T}$ (rows in $A$), and we denote it by $r_{F}(A)$.

\begin{definition}(Algebra isomorphism)~\cite{Xiong84Modern}
Let $(\mathbb{A}, \cdot )$ and $(\mathbb{B}, \circ)$ be two closed algebraic systems. If there exists a bijection $f$ from $A$ to $B$ such that
$f(A_{1}\cdot A_{2}) = f(A_{1}) \circ f(A_{2})$ for all $A_{1}, A_{2} \in  A$,
then we say $f$ is an isomorphism, and $A, B$ are isomorphic, denoted by $A \cong B$.
\end{definition}

\subsection{Matroid theory}
Matroid theory borrows extensively from the terminology of linear algebra theory and graph theory,
largely because it is the abstraction of various notions of central importance in these fields, such as independent set, base, rank function.
For convenience, we introduce some symbols firstly.

\begin{definition}\cite{Lai01Matroid,Oxley93Matroid}
Let $\mathcal{A}$ be a family of subsets of $U$. One can denote\\
$Low (\mathcal{A})=\{X \subseteq U: \exists A\in \mathcal{A}~such~that~A \subseteq X\}$;\\
$Min(\mathcal{A})=\{X \subseteq U: \forall A \in \mathcal{A}, if~Y \subseteq X, then~X=Y\}$;\\
$Max(\mathcal{A})=\{X \subseteq U: \forall A\in \mathcal{A}, if X \subseteq X, then~X=Y\}$.
\end{definition}

\begin{definition}(Matroid)~\cite{Lai01Matroid,Oxley93Matroid}
A matroid is an ordered pair $(U,\mathcal{I})$ consisting of a finite set $U$ and a collection $\mathcal{I}$ of subsets of $U$ satisfying
the following three conditions:\\
(I1) $\emptyset \in \mathcal{I}$;\\
(I2) If $I\in \mathcal{I}$ and $I^{'} \subseteq I$, then $I^{'}\in \mathcal{I}$;\\
(I3) If $I_{1},I_{2}\in \mathcal{I}$ and $|I_{1}|<|I_{2}|$,
then there is an element $e\in I_{2}-I_{1}$ such that $I_{1}\bigcup e\in \mathcal{I}$, where $|X|$ denotes the cardinality of $X$.
\end{definition}

Let $M(U,\mathcal{I})$ be a matroid.
The members of $\mathcal{I}$ are the independent sets of $M$.
A set in $\mathcal{I}$ is maximal, in the sense of inclusion, is called a base of the matroid $M$.
If $A\notin \mathcal{I}$, $A$ is called dependent set.
In the sense of inclusion, a minimal dependent subset of $U$ is called a circuit of the matroid $M$.
The collections of the bases, the dependent sets and the circuits of a matroid $M$ are denoted by $\mathcal{B}(M)$, $\mathcal{D}(M)$ and $\mathcal{C}(M)$, respectively.

Matroids can be defined in many different but equivalent ways.
The following definition defines a matroid from the viewpoint of circuit.

\begin{definition}(Circuit axiom)\cite{Lai01Matroid,Oxley93Matroid}
\label{P:circuitaxiom}
Let $\mathcal{C}$ be a family of subset of $U$.
There exists a matroid $M$ such that $\mathcal{C}=\mathcal{C}(M)$ if and only if $\mathcal{C}$ satisfies the following conditions:\\
(C1) $\emptyset \in \mathcal{C}$;\\
(C2) for all $C_{1}, C_{2}\in \mathcal{C}$, if $C_{1}\subseteq C_{2}$, then $C_{1}=C_{2}$;\\
(C3) for all $C_{1}, C_{2}\in \mathcal{C}$, if $C_{1}\neq C_{2}$ and $x\in C_{1} \bigcap C_{2}$, then there exists
$C_{3}\in \mathcal{C}$ such that $C_{3}\subseteq C_{1}\bigcup C_{2}-\{x\}$.
\end{definition}

The name "matroid" was coined by Whitney because a class of fundamental examples of such objects arises from matrices in the following way.

\begin{definition}(Vector matroid)\cite{Lai01Matroid,Oxley93Matroid}
Let $U$ be the set of column labels of an $m \times n$ matrix $A$ over a field $F$,
and $\mathcal{I}$ the set of subsets $X$ of $U$ for which the columns labeled by $X$ is linearly independent in the vector space $V(m, F)$.
Then $(E,\mathcal{I})$ is a matroid.
It is called the vector matroid of $A$, which denoted by $M_{F}[A]$.
\end{definition}

\begin{example}
\label{example1}
Let $\mathbf{R}$ be a real number field and $A$ a matrix over $\mathbf{R}$.
$$A=\bordermatrix[{[]}]{
& 1 & 2 & 3 & 4 & 5\cr
&  1    & 0     & 1     & 0    &0     \cr
&  0    & 1     & 0     & 1    &1     \cr
}$$
Then $U=\{1, 2, 3, 4, 5\}$ and $M_{\mathbf{R}}[A]=(E, \mathcal{I})$, where $\mathcal{I}=\{\emptyset, \{1\}, \{2\}, \{3\}, \{4\},$ $ \{5\}, \{1,2\},$ $\{1,4\}, \{1,5\}, $ $\{2,3\}, \{3,4\}, \{3,5\}\}$.
\end{example}
%

\begin{definition}(Isomorphism)\cite{Lai01Matroid,Oxley93Matroid}
Let $M_{1}=M(U_{1}, \mathcal{I}_{1})$ and $M_{2}=M(U_{2}, \mathcal{I}_{2})$ be two matroids.
$M_{1}$ and $M_{2}$ are isomorphic, denoted as $M_{1}\cong M_{2}$, if there is a bijection $\varphi:U_{1}\rightarrow U_{2}$ such that
$I\in \mathcal{I}_{1}$ if and only if $\varphi(I)\in \mathcal{I}_{2}$.
\end{definition}

\section{Matrix to matroidal structure of rough sets over a field}
\label{S:Representationofmatroidinducedbyanequivalencerelationoverafiled}
Matrix, which is a good computational tool and easy to represent, compute and accelerate, finds many applications in most scientific fields.
As an important branch of matroid theory, vector matroid, which is defined on the set of columns of matrix, provides good tool to study rough sets.
In this section, we will construct matroidal structures of rough sets over a field from matrices through vector matroid.
First, an existing matroidal structure of rough sets is provided, and a matrix representation of an equivalence relation is established.
Given a field, through proving the existing matroidal structure is the same as the one induced by the matrix representation through vector matroids, we construct a matroidal structure of rough sets over a field from matrix.
First of all, an approach to induce a matroidal structure from an equivalence relation is provided.

\begin{proposition}\cite{WangZhuZhuMin12Matroidal}
\label{matroidinducedbyenquivalencerelation}
Let $R$ be an equivalence relation on $U$ and $U/R=\{P_{1}, P_{2}, \cdots, P_{s}\}$.
\begin{center}
    $\mathcal{C}(R)=\{\{x,y\} \subseteq U|\{x,y\} \subseteq P_{i}, \forall i \in \{1, 2, \cdots, s\}\}$
\end{center}
satisfies circuit axiom $(C1), (C2)$ and $(C3)$.
Moreover, there exists a matroid $M$ such that $\mathcal{C}(M)=\mathcal{C}(R)$, and
we denote this matroid as $M(R)$.
\end{proposition}

The above proposition proposes an approach to induce a matroidal structure of rough sets from the viewpoint of set theory.
Matrix as a research tool has existed in most scientific fields.
We also want to use matrix to study rough sets.
Therefore, we define a matrix representation of an equivalence relation firstly.

\begin{definition}
\label{matrixinducedbyequivalencerelation}
Let $R$ an equivalence relation on $U=\{x_{1}, x_{2}, \cdots, x_{n}\}$ and $U/ R =\{P_{1}, P_{2}, \cdots, P_{s}\}$.
We denote a matrix $B(R)=(b_{ij})_{s\times n}$ as follows:\\
\begin{align}
b_{ij}=\left\{\begin{aligned}%
&1 && \mbox x_{j} \in P_{i},\\
&0 && \mbox x_{j} \notin P_{i}.\\
\end{aligned}\right.
\end{align}
\end{definition}

\begin{remark}
$B(R)$ does not contain zero rows and zero columns, and any column of it has only one non-zero component.
If we denote $1$ and $0$ as the multiplicative and additive identity elements of $F$, respectively, then $B(R)$ can be regarded as a matrix over $F$.
\end{remark}

For any element of $U$, we can denote it by $x_{i}(i\in \mathbf{N}_{+})$.
For any element of $U/R$, we can denote it by $P_{i}(i \in \mathbf{N}_{+})$.
In above definition, we label the columns of $B(R)$ by the elements of $U$ (in the sequential order of $\Gamma_{U}$) and the rows of $B(R)$ by the elements of $U/R$ (again, in the sequential order of $\Gamma_{U/R}$), where $\Gamma_{U}$ indicates the index set of all elements of $U$ and $\Gamma_{U/R}$ indicates the index set of all elements of $U/R$.
If the order of $\Gamma_{U}$(resp. $\Gamma_{U/R}$) changes, then $B(R)$ changes accordingly.
An example is provided to illustrate the statements.

\begin{example}
\label{example2}
Suppose $R$ is an equivalence relation on $U=\{x_{1}, x_{2}, x_{3}, x_{4}, x_{5}\}$ and $U/R=\{P_{1}, P_{2}\}$, where $P_{1} = \{x_{1},x_{3}\}$ and $P_{2} = \{x_{2},x_{4},x_{5}\}$.
Then $\Gamma_{U} = \{1,2,3,4,$ $5\}$(sequential order) and $\Gamma_{U/R} = \{1,2\}$.
We also can obtain one matrix representation of $R$ as follows:
$$B(R)=\bordermatrix[{[]}]{
& x_{1} & x_{2} & x_{3} & x_{4} & x_{5}\cr
P_{1}&  1    & 0     & 1     & 0    &0     \cr
P_{2}&  0    & 1     & 0     & 1    &1     \cr
}.$$
If $U = \{x_{1}, x_{3}, x_{2}, x_{4}, x_{5} \}$ and $U/R=\{P_{1}, P_{2}\}$, then $\Gamma_{U} = \{1, 3, 2, 4, 5\}$ and $\Gamma_{U/R} = \{1,2\}$.
We also can obtain the other matrix representation of $R$ as follows:
$$B(R)=\bordermatrix[{[]}]{
& x_{1} & x_{3} & x_{2} & x_{4} & x_{5}\cr
P_{1}&  1    & 1     & 0     & 0    &0     \cr
P_{2}&  0    & 0     & 1     & 1    &1     \cr
}.$$
If $U = \{x_{1}, x_{3}, x_{2}, x_{4}, x_{5} \}$ and $U/R=\{P_{2}, P_{1}\}$, then $\Gamma_{U} = \{1,2,3,4,$ $5\}$ and $\Gamma_{U/R} = \{2,1\}$.
We also can obtain another matrix representation of $R$ as follows:
$$B(R)=\bordermatrix[{[]}]{
& x_{1} & x_{3} & x_{2} & x_{4} & x_{5}\cr
P_{1}&  0    & 0     & 1     & 1    &1     \cr
P_{2}&  1    & 1     & 0     & 0    &0     \cr
}.$$
Therefore, different orders of $\Gamma_{U}$ (resp. $\Gamma_{U/R}$) determine different matrix representations of $R$.
\end{example}

As we know, a matrix with entries in a field gives rise to a matroid on its set of columns through vector matroids.
The dependent sets of the matroid are those columns of the matrix that are linearly dependent as vectors over the field.
Let $M=M_{F}[A]$.
In general, $M$ does not uniquely determine $A$.
One can obtain a matrix from $A$ by using some row elementary transformations which root in matroid theory.
It is not difficult to check that $M$ remains unchanged through these transformations.
In addition to that, if we interchange any two columns of $A$ with the labels of them, then $M$ remains unchanged.
Once the labels change but the columns labeled by them dose not change accordingly, then these two vector matroids may not be the same one.
In order to understand the above standpoints better, we take the following matrix for example.

\begin{example}
\label{example3}
Let us revisit Example \ref{example2}.
Suppose $U=\{x_{1}, x_{2}, x_{3}, x_{4}, x_{5}\}$ and $U/R=\{P_{1}, P_{2}\}$.
We obtain a matrix $B(R)$ (the first one in Example \ref{example2}).
If we interchange the $2th$ column and the $3th$ column of $B(R)$ but the labels of them remain unchanged, then we obtain the other matrix as follows:
$$B^{'}(R)=\bordermatrix[{[]}]{
& x_{1} & x_{2} & x_{3} & x_{4} & x_{5}\cr
&  1    & 1     & 0     & 0    &0     \cr
&  0    & 0     & 1     & 1    &1     \cr
}.$$
If we interchange the $2th$ column and the $3th$ column of $B(R)$ with the labels of them, then we obtain another matrix (the second one in Example \ref{example2}) and we denote it by $B^{''}(R)$.
It is clear that $M_{\mathbf{R}}[B^{'}(R)] = (U, \mathcal{I}^{'})$, where $\mathcal{I}^{'}=\{\emptyset, \{x_{1}\}, \{x_{2}\}, \{x_{3}\}, \{x_{4}\},$ $ \{x_{5}\},
\{x_{1},x_{3}\},$ $\{x_{1},x_{4}\}, \{x_{1},x_{5}\}, $ $\{x_{2},x_{3}\}, \{x_{2},x_{4}\}, \{x_{2},x_{5}\}\}$ and
$M_{\mathbf{R}}[B^{''}(R)]$ $= (U, \mathcal{I}^{''})$, where $\mathcal{I}^{''}=\{\emptyset, \{x_{1}\}, \{x_{2}\},$ $ \{x_{3}\},
\{x_{4}\},$ $ \{x_{5}\}, \{x_{1}, x_{2}\},$ $\{x_{1}, x_{4}\}, \{x_{1}, x_{5}\}, $ $\{x_{2}, x_{3}\}, \{x_{3}, x_{4}\},$ $ \{x_{3},x_{5}\}\}$.
One can define a mapping $\varphi:U \rightarrow U$ as follows: $\varphi(x_{2})=x_{3}$, $\varphi(x_{3})=x_{2}$ and $\varphi(x_{i})= x_{i}(i=1,4,5)$.
It is obvious that the mapping $\varphi$ is bijection.
Thus $M_{\mathbf{R}}[B(R)] =M_{\mathbf{R}}[B^{''}(R)] \cong M_{\mathbf{R}}[B^{'}(R)]$.
\end{example}

As the order of row labels of a matrix dose not change the vector matroid induced by the matrix, we will not take them into consideration in the following discussion.
For a ground set $U$, we denote the $ith$ element of $U$ by $x_{i}$ ($i\in \mathbf{N}_{+}$ and $ 1 \leq i \leq |U|$) and we obtain $\Gamma_{U}$.
For an order of $\Gamma_{U}$, we obtain one matrix representation of an equivalence relation.
Changing the order of $\Gamma_{U}$, we obtain the other matrix.
Essentially, the latter matrix is obtained from the former matrix through exchanging some columns with the labels of them.
Therefore, for arbitrary order of $\Gamma_{U}$, we obtain different $B(R)$.
However, these matrices induce the same matroid over the same field.
In a word, in order to study the relation between the matroid and $M(R)$, we just need to study the relation between $M(R)$ and the vector matroid induced by $B(R)$ which defined in Definition \ref{matrixinducedbyequivalencerelation}.
The following theorem indicates that the matroid is the same as $M(R)$ over any field.

\begin{theorem}
\label{T: circuitsbetweenmatroidandvectormatroid}
Let $R$ be an equivalence relation on $U$ and $B(R)$ a matrix representation of $R$ over $F$.
$M(R) = M_{F}[B(R)]$.
\end{theorem}

\begin{proof}
Let $U=\{x_{1}, x_{2}, \cdots, x_{n}\}$.
We can obtain $B(R) =[\beta_{1}, \beta_{2}, \cdots, $ $\beta_{n}]$ and we know the columns of $B(R)$ are labeled, in order, by $x_{1}, x_{2}, \cdots, x_{n}$.
Since a set of circuits decides only one matroid, we need to prove only $\mathcal{C}(R)=\mathcal{C}(M_{F}[B(R)])$, i.e., $\mathcal{C}(R)=Min (\{X \subseteq U:$ The columns of $B(R)$ labeled by $X$ are linearly dependent in $V(n, F)\})$.
Suppose $U/R=\{P_{1}, P_{2}, \cdots, $ $P_{s}\}$.
$\forall \{x_{i}, x_{j}\} \in \mathcal{C}(R)$, there exists $P_{k}\in U/R$ such that $\{x_{i}, x_{j}\} \subseteq P_{k}$.
According to the definition of $B(R)$, we know that $\beta_{i} = \beta_{j}$.
Thus $\beta_{i}$ and $\beta_{j}$ are linearly dependent in $V(n, F)$.
Because $B(R)$ dose not contain zero columns, the column labeled by $x_{i}$ or $x_{j}$ is linearly dependent in $V(n, F)$.
Hence $\mathcal{C}(R)\subseteq Min(\{X \subseteq U:$ The columns of $B(R)$ labeled by $X$ are linearly dependent in $V(n, F)\})$.
For all $X\in $ $Min(\{$ $X \subseteq U:$ The columns of $B(R)$ labeled by $X$ are linearly dependent in $V(n, F)$ $\})$, then $X$ is a
dependent set in $M(R)$;
otherwise, $X$ dose not contain circuits.
According to the definition of $B(R)$, we know the columns of $B(R)$ labeled by $X$ are different and these columns form a $|X| \times |X|$ identity matrix, where $|X|$ denotes the cardinality of $X$.
Hence these columns are linearly independent in $V(n,F)$, which implies contradictory.
Now we need to prove $X$ is a circuit of $M(R)$.
From above proof, we know $X$ is circuit of $M(R)$; otherwise, there exists $C \in \mathcal{C}(M(R))$ such that $C \subset X$, i.e., $C \in \{X \subseteq U:$ The columns of $B(R)$ labeled by $X$ are linearly dependent in $V(n, F)\}$ such that $C \subset X$ because $\mathcal{C}(R)\subseteq Min(\{X \subseteq U:$ The columns of $B(R)$ labeled by $X$ are linearly dependent in $V(n, F)\})$.
That contradicts the fact that $ X\in Min(\{$ $X \subseteq U:$ The columns of $B(R)$ labeled by $X$ are linearly dependent in $V(n, F)$ $\})$.
Hence, $Min \{X \subseteq U:$ The columns of $B(R)$ labeled by $X$ are linearly dependent in $V(n, F)\} \subseteq \mathcal{C}(R)$.
\end{proof}

As we know, a common field is finite field or Galois field which contains a finite number of elements.
Now we introduce the simplest field, i.e., binary field.

\begin{definition}(Binary field)\cite{Xiong84Modern}
Let $GF=\{0,1\}$.
If the addition and multiplication of $GF$ are defined in Table 1, then $(GF,+,\cdot)$ is called binary field and we denote it as $GF(2)$.
\end{definition}

\begin{table}
 \caption{Addition and multiplication of $GF$}
 \centering
 \subtable{
\begin{tabular}{c|cc}
$+$~~~ & ~~0~~~~~& 1 \\
\hline
$0$~~~ & ~~0~~~~~& 1 \\
$1$~~~& ~~ 1~~~~~& 0 \\
\end{tabular}
        \label{tab:firsttable}
 }
 \qquad
 \subtable{
\begin{tabular}{c|cc}
$\cdot$~~~ & ~~0~~~~~& 1 \\
\hline
$0$~~~ & ~~0~~~~~& 0 \\
$1$~~~ & ~~ 0~~~~~& 1 \\
\end{tabular}
\label{tab:secondtable}
 }
 \end{table}

$GF(2)$ is a special field.
Based on Theorem~\ref{T: circuitsbetweenmatroidandvectormatroid}, it is not difficult for us to obtain the following corollary.

\begin{corollary}
Let $R$ be an equivalence relation on $U$. $M(R) = M_{GF(2)}[B(R)]$.
\end{corollary}

For an order of $\Gamma_{U}$, we can obtain a matrix $B(R)$.
From above theorem, we find that the matrix induces the same matroid over different fields, which is determined by the particularity of the matrix.
However, in many cases, the vector matroids induced by the same matrix over different fields may not be the same one.
The example below illustrates this viewpoint.

\begin{example}
Suppose
\begin{center}
$A = \left[\begin{array}{cccccc}
 1    ~&~ 0     ~&~ 0     ~&~1     ~&~ 1    ~&~ 0 \\
 0    ~&~ 1     ~&~ 0     ~&~1     ~&~ 0    ~&~ 1 \\
 0    ~&~ 0     ~&~ 1     ~&~0     ~&~ 1    ~&~ 1
\end{array}\right]$.
\end{center}
We may as well suppose $A=[\alpha_{1}, \alpha_{2}, \alpha_{3}, \alpha_{4}, \alpha_{5}, \alpha_{6}]$ which are labeled, in order, by $x_{1}, x_{2}, x_{3}, x_{4}, x_{5}, x_{6}$.
Then $\alpha_{4}, \alpha_{5}, \alpha_{6}$ are linearly independent over real number field because $det([\alpha_{4}, \alpha_{5}, \alpha_{6}])=-2 \neq 0$.
However, $det([\alpha_{4}, \alpha_{5}, \alpha_{6}])=0$ over $GF(2)$, i.e., $\alpha_{4}, \alpha_{5}, \alpha_{6}$ are linearly dependent over binary field.
In addition to that, any two columns of $\alpha_{4}, \alpha_{5}$ and $\alpha_{6}$ are linearly independent over $GF(2)$ because the determinant of them is not zero over $GF(2)$.
Therefore, $\{x_{4},x_{5},x_{6}\} \in \mathcal{C}(M_{GF(2)}[A])$ but $\{x_{4},x_{5},x_{6}\} \notin \mathcal{C}(M_{\mathbf{R}}[A])$, i.e., $M_{GF(2)}[A] \neq M_{\mathbf{R}}[A]$.
\end{example}

\section{Matroidal structure of rough sets over a field to two special matrix solution spaces }
\label{S:Nullspaceapproachtoroughsetthroughmatroid}
In section \ref{S:Representationofmatroidinducedbyanequivalencerelationoverafiled}, we has obtained a matroidal structure of rough sets over a field from matrix.
In this section, we will study some characteristics of the matroidal structure through two matrix solution spaces, especially matrix null space.
First of all, an operator is proposed to connect vector space with set theory.

\begin{definition}\cite{Lai01Matroid,Oxley93Matroid}
Let $F$ be a field and $U=\{x_{1}, x_{2}, \cdots, x_{n}\}$. We define a mapping $\theta: V(n, F) \rightarrow 2^{U}$ as follows:
for all $\mathbf{v}=(v_{1}, v_{2}, \cdots, v_{n})^{T} \in V(n, F)$, $\theta(\mathbf{v})=\{x_{i} \in U: v_{i} \neq 0, 1 \leq i \leq n\}$, where $0$ is the additive identity element of $F$. we call $\theta(\mathbf{v})$ the support of $\mathbf{v}$.
\end{definition}

\begin{example}
Let $F$ be a field and $U = \{x_{1}, x_{2}, x_{3}, x_{4}, x_{5}\}$. If $\mathbf{v}_{1}=(1, -1, 3, 0, 4 )^{T} \in V(5, F)$ and $\mathbf{v}_{2}=(1, 1, 0, 0, 1)^{T} \in V(5, F)$,
then $\theta(\mathbf{v}_{1}) = \{x_{1}, x_{2}, x_{3}, x_{5}\}$ and $\theta(\mathbf{v}_{2}) = \{x_{1}, x_{2}, x_{5}\}$.
\end{example}

Matrix null space is an important concept in linear algebra theory.
According to the characteristic of it and the definition of vector matroid, it is natural for us to combine them with each other.

\begin{proposition}
\label{P:nullspaceanddependentset}
Let $F$ be a field and $A$ an $m\times n$ matrix over $F$. If $M=M_{F}[A]$, then
$\{\theta(\mathcal{N}_{F}(A))-\emptyset\} \subseteq \mathcal{D}(M)$.
\end{proposition}

\begin{proof}
It is well known that, by row elementary transformations and some column elementary transformations which root in matroid theory, one can reduce any matrix $G$ to the form $[I_{r}| D]$, where $r = r_{F}(G)$, $I_{r}$ is the $r \times r$ identity matrix and $D$ is a $r \times (n-r)$ matrix over $F$.
Then we may as well suppose $A=[I_{r}| D]=[\alpha_{1}, \alpha_{2}, \cdots, \alpha_{n}]$.
And the columns of $A$ are labeled, in order, by $x_{1}, x_{2}, \cdots, x_{n}$.
If $r=n$, then $A=I_{n}$.
Since $r_{F}(A) = n$.
Thus $D(M)=\emptyset$, and $A \mathbf{v}=\mathbf{0}$ has only trivial solution, i.e, $\{\theta(\mathcal{N}_{F}(A))-\emptyset\} = \emptyset$.
Hence we obtain the result.
If $r < n$, then $A \mathbf{v}=\mathbf{0}$ has nontrivial solution, i.e., $\{\theta(\mathcal{N}_{F}(A))-\emptyset\} \neq \emptyset$.
For all $D = \{x_{i_{1}}, x_{i_{2}}, \cdots, x_{i_{s}}\} \in \{\theta(\mathcal{N}_{F}(A))-\emptyset\}$, there exists $\mathbf{v}_{D} = (v_{1}, v_{2}, \cdots, v_{n})^{T} \in V(n, F)$ such that
$A \mathbf{v}_{D} = \mathbf{0}$ and $\theta(\mathbf{v}_{D})=D$.
Then $v_{i_{1}}, v_{i_{2}}, \cdots, v_{i_{s}}(1 \leq s \leq n)$ are non-zero components of vector $\mathbf{v}_{D}$.
Thus $\mathbf{0} = A \mathbf{v}_{D} = \sum_{i=1}^{n} v_{i} \alpha_{i}= \sum_{j=1}^{s} v_{i_{j}} \alpha_{i_{j}}$, that is,
the columns of $A$ labeled by $D$ are linearly dependent over $F$.
Hence, $\{\theta(\mathcal{N}_{F}(A))-\emptyset\} \subseteq \mathcal{D}(M)$.
\end{proof}

Conversely, the collection of dependent sets of $M_{F}[A]$ may not be contained in $\{\theta(\mathcal{N}_{F}(A))$ $-\emptyset\}$.
The following example illustrates that viewpoint.

\begin{example}
Let us revisit Example \ref{example2}.
We may as well suppose $B(R)=[\mathbf{\beta}_{1}, \mathbf{\beta}_{2}, \mathbf{\beta}_{3}, \mathbf{\beta}_{4},$ $\mathbf{\beta}_{5}]$ (the first one in Example \ref{example2}).
We know that the columns of $B(R)$ are labeled, in order, by $x_{1}, x_{2}, x_{3}, x_{4}, x_{5}$.
It is easy to check that $\mathbf{\beta}_{1}, \mathbf{\beta}_{2}, \mathbf{\beta}_{4}$ are linearly dependent in $V(n, GF(2))$, that is, $\{x_{1}, x_{2}, x_{4}\}$ $ \in \mathcal{D}(M_{GF(2)}[B(R)])$.
Suppose $\mathbf{x}=(1, 1, 0, 1, 0)^{T}\in V(5, GF(2))$. We know that $\{x_{1}, x_{2}, x_{4}\}=\theta(\mathbf{x})$, but $B(R)\mathbf{x}=(1,0)^{T} \neq \mathbf{0}$ over binary field.
Thus $\{x_{1}, x_{2}, x_{4}\} \notin \{\theta(\mathcal{N}_{GF(2)}(B(R)))-\emptyset\}$.
\end{example}

What about the relation between the matrix null space and the circuits of a vector matroid?
In order to solve this problem, we present the following proposition firstly.

\begin{proposition}
\label{theminimalityofset}
Let $\mathcal{F}$ and $\mathcal{S}$ are two families of subsets of $U$.
If $\mathcal{F} \subseteq \mathcal{S}$ and $Min(\mathcal{S})\subseteq Min(\mathcal{F})$, then $Min(\mathcal{S}) = Min(\mathcal{F})$.
\end{proposition}

\begin{proof}
We need to prove $Min(\mathcal{F}) \subseteq Min(\mathcal{S})$.
If $Min(\mathcal{F}) \nsubseteq Min(\mathcal{S})$, then there exists $F \in Min(\mathcal{F}) - Min(\mathcal{S})$.
Since $\mathcal{F} \subseteq \mathcal{S}$, $Min(\mathcal{F}) \subseteq \mathcal{S}$, i.e, $F \in Min(\mathcal{F}) - Min(\mathcal{S}) \subseteq \mathcal{S} - Min(\mathcal{S})$.
Thus there exists $S \in \mathcal{S}$ such that $S \subset F$.
Denote $\mathcal{W} =Min \{W \in \mathcal{S}: W \subset F\}$.
For all $W \in \mathcal{W}$, we know $W \in Min (\mathcal{S})$; otherwise, there exists $W_{1} \in \mathcal{S}$ such that $W_{1} \subset W$.
Since $W \subseteq F$, $W_{1} \subset F$.
Thus there exists $W_{1} \in \{W \in \mathcal{S} : W \subset F\}$ such that $W_{1} \subset W$, which contradicts $W \in \mathcal{W}$.
Therefore we have $W \in Min(\mathcal{F})$ according to $Min(\mathcal{S}) \subseteq Min(\mathcal{F})$, which contradict the fact that $F \in Min(\mathcal{F})$.
Hence we prove $Min(\mathcal{F}) \subseteq \mathcal{A}$, that is, $\mathcal{A} = Min(\mathcal{F})$.
\end{proof}

\begin{remark}
For any two family of sets $\mathcal{F}$ and $\mathcal{S}$, if $\mathcal{F} \subseteq \mathcal{S}$, then we may not have the result $Min(\mathcal{F}) \subseteq Min(\mathcal{S})$.
\end{remark}

\begin{example}
Let $\mathcal{F} = \{\{2,3\}\}$ and $\mathcal{S} = \{\{2\}, \{3\}, \{2,3\}\}$.
It is obvious that $\mathcal{F} \subseteq \mathcal{S}$, but $Min{F} = \{\{2,3\}\} \nsubseteq Min(S) = \{\{2\}, \{3\}\}$.
\end{example}

The following theorem indicates that, over a field, the collection of circuits of the vector matroid induced by a matrix is just the collection of the minimal non-empty sets that are supports of members of null space of the matrix.

\begin{theorem}
\label{circuitandresolutionspace}
Let $F$ be a field and $A$ an $m\times n$ matrix over $F$. If $M=M_{F}[A]$, then $\mathcal{C}(M)=Min(\{\theta(\mathcal{N}_{F}(A))-\emptyset\})$.
\end{theorem}

\begin{proof}
It is well known that, by row elementary transformations and some column elementary transformations which root in matroid theory, one can reduce matrix $G$ to the form $[I_{r}| D]$, where $r = r_{F}(G)$, $I_{r}$ is the $r \times r$ identity matrix and $D$ is a $r \times (n-r)$ matrix over $F$.
Then we may as well suppose $A=[I_{r}| D]=[\mathbf{\alpha}_{1}, \mathbf{\alpha}_{2}, \cdots, \mathbf{\alpha}_{n}]$ of which the columns are labeled, in order, by $x_{1}, x_{2}, \cdots, x_{n}$.
If $r=n$, then $\mathcal{C}(M)=\emptyset$ and $A \mathbf{v}=\mathbf{0}$ has only trivial solution, i.e, $\{\theta(\mathcal{N}_{F}(A))-\emptyset\} = \emptyset$.
Thus we obtain the result.
If $r < n$, then $A \mathbf{v}=\mathbf{0}$ has nontrivial solution, i.e., $\{\theta(\mathcal{N}_{F}(A))-\emptyset\} \neq \emptyset$.
We may as well suppose $C=\{x_{i_{1}}, x_{i_{2}}, \cdots, x_{i_{s}}\}$ is a circuit of $M$, where $1 \leq s \leq n$.
According to the definition of vector matroid, we know the columns of $A$ which are labeled by the elements of $C$ are linearly dependent, that is, there exist some elements $k_{i_{1}}, k_{i_{2}}, \cdots, k_{i_{s}}$ of $F$ which all are not equal to zero such that $\mathbf{0} = k_{i_{1}} \alpha_{i_{1}} + k_{i_{2}} \alpha_{i_{2}} + \cdots + k_{i_{s}} \alpha_{i_{s}}=
\sum_{j=1}^{s} k_{i_{j}}\alpha_{i_{j}} + \sum_{t\neq i_{j}, 1 \leq j \leq s} 0 \cdot \alpha_{t}=A \mathbf{v}$,
where $\mathbf{v}=(v_{1}, v_{2}, \cdots, v_{n})^{T}\in V(n,F)$.
If $p \neq i_{j}(1 \leq j \leq s)$, then $v_{p}=0$.
Of course, if $p = i_{j}(1 \leq j \leq s)$, then $v_{p}$ may equal to $0$.
Next, we want to prove that $k_{i_{j}} \neq 0$ for all $j \in \{1, 2, \cdots, s\}$; otherwise, we may as well suppose $k_{i_{1}}, k_{i_{2}}, \cdots, k_{i_{t}} \neq 0 (t < s)$, then $\theta(\mathbf{v}^{'}) \subset C$, where $\mathbf{v}^{'} = (v_{1}^{'}, v_{2}^{'}, \cdots, v_{n}^{'})^{T}$, and $ v_{p}^{'} \neq 0$ if and only if $p\in \{i_{1}, i_{2}, \cdots, i_{t}\}$.
Moreover, $\mathbf{0} = k_{i_{1}} \alpha_{i_{1}} + k_{i_{2}} \alpha_{i_{2}} + \cdots + k_{i_{s}} \alpha_{i_{s}}=k_{i_{1}}\alpha_{i_{1}} + k_{i_{2}}\alpha_{i_{2}} + \cdots + k_{i_{t}}\alpha_{i_{t}}$,
then the columns $\alpha_{i_{1}}, \alpha_{i_{2}}, \cdots, \alpha_{i_{t}}$ are linearly dependent in $V(n,F)$, hence $\theta(\mathbf{v}^{'})$ is a dependent set of $M$.
According to the definition of dependent set of matroid, there exists $C_{1} \in \mathcal{C}(M)$ such that $C_{1} \subseteq \theta(\mathbf{v}^{'}) \subset C$.
According to the (2) of circuit axiom, we can obtain the contradictory.
Hence, $\theta(\mathbf{v})=C$, that is, $C \in \{\theta(\mathcal{N}_{F}(A))-\emptyset\}$.
Moreover, $C \in Min(\{\theta(\mathcal{N}_{F}(A))-\emptyset\})$; otherwise, there exists $D \in \{\theta(\mathcal{N}_{F}(A))-\emptyset\}$ such that $D \subset C$.
According to Proposition \ref{P:nullspaceanddependentset}, we have $\{\theta(\mathcal{N}_{F}(A))-\emptyset\} \subseteq \mathcal{D}(M)$.
Thus $D\in \mathcal{D}(M)$, that is, there exists $C_{2} \in \mathcal{C}(M)$ such that $C_{2} \subseteq D \subset C$ which contradicts the (2) of circuit axiom.
Therefore $Min (\mathcal{D}(M)) = \mathcal{C}(M) \subseteq Min \{\theta(\mathcal{N}_{F}(A))-\emptyset\}$.
Combing with $\{\theta(\mathcal{N}_{F}(A))-\emptyset\} \subseteq \mathcal{D}(M)$ and Proposition \ref{theminimalityofset}, we can obtain $Min (\mathcal{D}(M))= Min (\{\theta(\mathcal{N}_{F}(A))-\emptyset\})=\mathcal{C}(M))$.
\end{proof}

As we know, over the same field, any two matrices which have the same number of columns may generate the same vector matroid.
Theorem \ref{circuitandresolutionspace} indicates that the families of the minimal non-empty sets that are supports of members of null space of these matrices are unique.
Combining with Theorem \ref{T: circuitsbetweenmatroidandvectormatroid} and \ref{circuitandresolutionspace},
we obtain the following two corollaries.

\begin{corollary}
\label{circuitandresolutionspaceofequivalencerelation}
$\mathcal{C}(R)=Min(\{\theta(\mathcal{N}_{F}(B(R)))-\emptyset\})$.
\end{corollary}
%

For a field, the addition and multiplication of it are uncertain, we can not compute the circuits of a matroid easily.
Binary field is the simplest field and the operations of it are clear.
Hence we can obtain the circuits of $M(R)$ easily through calculating null space of $B(R)$ over the field.

\begin{corollary}
\label{nullspaceandcirciuitoverbinaryfield}
$\mathcal{C}(R)=Min(\{\theta(\mathcal{N}_{GF(2)}(B(R)))-\emptyset\})$.
\end{corollary}

According to the particularity of binary field, we may consider whether we can characterize $M(R)$ by using other solution spaces.
Inspired by the null space, we introduce the other matrix solution space to study $M(R)$.

\begin{definition}
Let $F$ be a field and $A$ an $m \times n$ matrix over $F$.
We denote the the set of all solutions to the non-homogeneous equation $A\mathbf{x}=\mathbf{1}$ by $I_{F}(A)$.
In set notion, $I_{F}(A)=\{\mathbf{x}\in V(n, F): A \mathbf{x}=\mathbf{1}\}$.
\end{definition}

So how do you use the space to characterize matroid $M(R)$ over binary field?
Firstly, we establish the relation between $\{\theta(I_{GF(2)}(B(R)))\}$ and the family of sets of which the upper approximations are equal to ground set.

\begin{proposition}
\label{anotherspaceandroughset}
Let $R$ be an equivalence relation on $U$ and $B(R)$ a matrix representation of $R$.
$\{\theta(I_{GF(2)}(B(R)))\} \subseteq \{X \subseteq U: R^{\ast}(X)=U\}$.
\end{proposition}

\begin{proof}
Suppose $U=\{x_{1}, x_{2}, \cdots, x_{n}\}$ and $U/R=\{P_{1}, P_{2}, \cdots, P_{s}\}$.
We can obtain $B(R)=[\beta_{1}, \beta_{2}, \cdots, \beta_{n}]$ of which the columns are labeled, in order, by $x_{1}, x_{2}, \cdots, x_{n}$.
For all $X \in \{\theta(I_{GF(2)}(B(R)))\}$, we may as well suppose $X=\{x_{i_{1}}, x_{i_{2}}, \cdots, x_{i_{t}}\}$ and
$\mathbf{v} \in V(n, GF(2))$ such that $\theta(\mathbf{v})=X$.
Then we can obtain $\mathbf{\beta}_{i_{1}} + \mathbf{\beta}_{i_{2}} + \cdots + \mathbf{\beta}_{i_{t}}= \mathbf{1}$.
Assume $A = [\mathbf{\beta}_{i_{1}}, \mathbf{\beta}_{i_{2}}, \cdots, \mathbf{\beta}_{i_{t}}]$.
In order to satisfy the equality $\mathbf{\beta}_{i_{1}} + \mathbf{\beta}_{i_{2}} + \cdots + \mathbf{\beta}_{i_{t}}= \mathbf{1}$ over binary field, the $jth(j \in \{1, 2, \cdots, s\})$ row of $A$ has odd number of $1$.
Combining with the definition of $B(R)$, we have $X \bigcap P_{i} \neq \emptyset$ for all $i\in \{1, 2, \cdots, s\}$, that is, $R^{\ast}(X)=U$.
Thus $X \subseteq \{X \subseteq U: R^{\ast}(X)=U\}$, that is, $\{\theta(I_{GF(2)}(B(R)))\} \subseteq \{X: R^{\ast}(X)=U\}$.
\end{proof}

However, there exists a subset $X$ of $U$ satisfies $R^{\ast}(X) = U$ but $B(R)\mathbf{v} \neq \mathbf{1}$ over binary field, where $\mathbf{v} \in V(n, GF(2))$ and $\theta(\mathbf{v}) = X$.

\begin{example}
Let us revisit Example 2.
We take the first one matrix $B(R)$ for example.
Suppose $X=\{x_{1}, x_{2}, x_{3}\}$.
Then $\mathbf{v}=(1,1,1,0,0)^{T}$ $ \in V(5,GF(2))$ and $\theta(\mathbf{v})=X$.
It is clear that $R^{\ast}(X)=U$, but $B(R)\mathbf{v}= \mathbf{\beta}_{1}+ \mathbf{\beta}_{2} + \mathbf{\beta}_{3}= (0,1)^{T}$, i.e., $X \notin \{\theta(I_{GF(2)}(B(R)))\}$ over binary field.
Hence $\{X: R^{\ast}(X)=U\} \nsubseteq \{\theta(I_{GF(2)}(B(R)))\}$.
\end{example}

In order to characterize $M(R)$ by using $I_{GF(2)}(B(R))$, we need to search a certain characteristic of the matroid which has close relation with the set $\{X \subseteq U: R^{\ast}(X)=U\}$.
According to the peculiarity of $B(R)$, the following proposition establishes the bases of $M(R)$.

\begin{proposition}
\label{baseofM(R)}
Let $R$ be an equivalence relation on $U$ and $U/R = \{P_{1}, P_{2}, \cdots, P_{s}\}$.
$\mathcal{B}(M(R)) = \{X \subseteq U: |X \bigcap P_{i}| = 1, \forall i\in \{1, 2, \cdots, s\}\}$.
\end{proposition}

\begin{proof}
Let $U=\{x_{1}, x_{2}, \cdots, x_{n}\}$.
we obtain the $B(R)=[\beta_{1}, \beta_{2},$ $ \cdots, \beta_{n}]$ of which the columns are labeled, in order, by $x_{1}, x_{2}, \cdots, x_{n}$.
According to the definition of matrix $B(R)$, we know $r_{F}(B(R))=s$.
Since $M(R) = M_{F}([B(R)])$, $B$ is a base of $M(R)$ if and only if the columns of $B(R)$ labeled by the elements of $B$ is a maximal independent subset of the matrix.
Since $r_{F}(B(R))=s$, any linearly independent columns of $B(R)$ with the cardinality $s$ form a maximal independent subset of $B(R)$.
For all $B\in \mathcal{B}(M(R))$, we may as well suppose $B=\{x_{i_{1}}, x_{i_{2}}, \cdots, $ $x_{i_{s}}\}$.
Then $x_{i_{j}} \in P_{j}$ for all $j\in \{1, 2, \cdots, s\}$;
otherwise, we may as well suppose there exist $x_{i_{1}}$ and $x_{i_{2}}$ such that $x_{i_{1}}, x_{i_{2}}\in P_{1}$,
then $\beta_{i_{1}}=\beta_{i_{2}}$.
Thus $\beta_{i_{1}}$ and $\beta_{i_{2}}$ are linearly dependent in $V(n,F)$, which makes the columns labeled by $B$ are linearly dependent in $V(n,F)$.
That contradicts $B \in \mathcal{B}(M([B(R)])$.
Hence $|B\bigcap P_{i}| = 1$ for all $i\in \{1,2,\cdots, s\}$, i.e., $\mathcal{B}(M(R)) \subseteq \{X \subseteq U: |X \bigcap P_{i}| = 1, \forall i\in \{1, 2, \cdots, s\}\}$.
Conversely, $\forall X\in \{X \subseteq U: |X \bigcap P_{i}| = 1, \forall i\in \{1, 2, \cdots, s\}\}$, then the cardinality of $X$ is $s$ and the columns of $B(R)$ labeled by $X$ are linearly independent over $F$, hence $X \in \mathcal{B}(M_{F}[B(R)])$.
According to Theorem \ref{T: circuitsbetweenmatroidandvectormatroid}, We have $\mathcal{B}(M_{F}[B(R)]) = \mathcal{B}(M(R))$.
Thus we obtain the result.
\end{proof}

According to the relation between the independent sets and the bases of a matroid, we obtain the following corollary.

\begin{corollary}
Let $R$ be an equivalence relation on $U$ and $U/R = \{P_{1}, P_{2}, \cdots, P_{s}\}$.
$\mathcal{I}(M(R)) = \{X \subseteq U: |X \bigcap P_{i}| \leq 1, \forall i\in \{1, 2, \cdots, s\}\}$.
\end{corollary}

Combing Proposition \ref{baseofM(R)} with the definition of upper approximation of any subset of ground set, we can obtain the following proposition.

\begin{proposition}
\label{anotherformofbaseofM(R)}
Let $R$ be an equivalence relation on $U$.
$\mathcal{B}(M(R))= Min (\{X \subseteq U: R^{\ast}(X)$ $=U\})$.
\end{proposition}

Based on the above three propositions, the following theorem connects the bases of $M(R)$ with the solution space $I_{GF(2)}(B(R))$.
It is interesting to find that, over binary field, the collection of bases of the vector matroid induced by a matrix is just the collection of the minimal non-empty sets that are supports of members of $I_{GF(2)}(B(R))$.

\begin{theorem}
\label{baseandanotherspace}
Let $R$ be an equivalence relation on $U$ and $B(R)$ a matrix representation of $R$.
$\mathcal{B}(M(R))=Min(\{\theta(I_{GF(2)}(B(R)))\})$.
\end{theorem}

\begin{proof}
Let $U=\{x_{1}, x_{2}, \cdots, x_{n}\}$ and $U/R=\{P_{1}, P_{2}, \cdots, P_{s}\}(s \leq n)$.
we obtain the $B(R)=[\beta_{1}, \beta_{2},$ $ \cdots, \beta_{n}]$ of which the columns are labeled, in order, by $x_{1}, x_{2}, \cdots, x_{n}$.
For all $B\in \mathcal{B}(M(R))$, we may as well suppose $B=\{x_{i_{1}}, x_{i_{2}}, \cdots, $ $x_{i_{s}}\}$.
According to Proposition \ref{baseofM(R)}, we know that $x_{i_{j}} \in P_{j}$ for all $j\in \{1, 2, \cdots, s\}$.
Let $\mathbf{v}=(v_{1}, v_{2}, \cdots, v_{n})^{T} \in V(n, GF(2))$, where $v_{i}=1$ if and only if $x_{i} \in B$.
Then $\theta(\mathbf{v})=B$.
According to the definition of $B(R)$, we know $B(R)\mathbf{v}=\mathbf{1}$.
Thus $\mathcal{B}(M(R)) \subseteq \{\theta(I_{GF(2)}(B(R)))\}$.
Next, we prove the minimality of $B$.
If $B \notin Min(\{\theta(I_{GF(2)}(B(R)))\})$,
then there exists $B_{1}\in \{\theta(I_{GF(2)}(B(R)))\}$ such that $B_{1} \subset B$.
However, for all $x_{i_{j}}\in B$, $B-\{x_{i_{j}}\} \notin \{\theta(I_{GF(2)}(B(R)))\}$ which implies contradiction.
Hence, we have $\mathcal{B}(M(R)) \subseteq Min(\{\theta(I_{GF(2)}(B(R)))\})$.
Combining with Proposition \ref{theminimalityofset}, \ref{anotherspaceandroughset} and \ref{anotherformofbaseofM(R)}, we have $\mathcal{B}(M(R))=Min(\{\theta(I_{GF(2)}(B(R)))\})$.
\end{proof}

As we know, any independent set of matroid $M(R)$ are those columns of $B(R)$ that are linearly independent as vectors over a field.
The following proposition establishes another representation of the independent sets of $M(R)$ by using the solution space $I_{GF(2)}(B(R)))\})$.

\begin{proposition}
Let $R$ be an equivalence relation on $U$ and $B(R)$ a matrix representation of $R$.
$\mathcal{I}(M(R))= Low(Min(\{\theta(I_{GF(2)}(B(R)))\}))$.
\end{proposition}

Characterizing the circuits, the bases and the independent sets of the matroid induced by an equivalence relation by matrix approaches lays the sound foundation for us to study the characteristics of the matroidal structure of rough sets from the viewpoint of matrix.

\section{Matrix null space over binary field to rough sets }
\label{Equivalencerelationinducedbymatrix}
In section \ref{S:Representationofmatroidinducedbyanequivalencerelationoverafiled}, we has obtained a matroidal structure of rough sets from matrix.
Over any field, the matroid induced by an equivalence relation is the one induced by a matrix representation of the equivalence relation.
Section \ref{S:Nullspaceapproachtoroughsetthroughmatroid} has studied certain characteristics of the matroid through matrix solution space such as null space.
In this section, we study how to construct an equivalence relation from matrix null space over binary field, and establish an isomorphism from a family of equivalence relations to a family of sets which any member is a collection of the minimal non-empty sets that are supports of members of null space of a binary dependence matrix.

In the following discussion, for any $m \times n$ matrix $A$, we suppose the columns of it are, in order, labeled by $x_{1}, x_{2}, \cdots, x_{n}$, and we denote the collection of the column labels as $U$.
Over binary field, we define a relation on $U$ by the null space of matrix as follows.

\begin{definition}
\label{relationinducedbymatroid}
Let $F$ be a field and $A$ be an $m \times n$ matrix over $F$.
One can define a relation $R_{F}(A)$ on $U$ as follows: for all $x_{i}, x_{j} \in U$,
\begin{center}
$(x_{i}, x_{j}) \in R_{F}(A)  \Leftrightarrow  x_{i}=x_{j}$ or $\mathbf{e}_{i}+\mathbf{e}_{j} \in \mathcal{N}_{F}(A)$,
\end{center}
where $U$ is a collection of column labels of $A$ and $\mathbf{e}_{i}, \mathbf{e}_{j}\in V(n, F)$ satisfy $\theta(\mathbf{e}_{i})= \{x_{i}\}$ and $\theta(\mathbf{e}_{j})=\{x_{j}\}$.
\end{definition}

\begin{remark}
If $F$ is a binary field, then $\forall x_{i}, x_{j}\in U$ and $x_{i}R_{F}(A)x_{j}$ implies $x_{i}=x_{j}$ or the columns of $A$ labeled by $x_{i}$ and $x_{j}$, respectively, are equivalent.
\end{remark}

\begin{example}
Suppose
$$A=\bordermatrix[{[]}]{
& x_{1} & x_{2} & x_{3} \cr
&1    & -1     & 1     \cr
&1    & -1     & 1     \cr
}.$$
According to the above definition, we know $U = \{x_{1}, x_{2}, x_{3}\}$ and $(x,x)\in R$ for all $x \in U$.
Since $\mathbf{e}_{1} = (1, 0, 0)^{T}$, $\mathbf{e}_{2}=(0, 1, 0)^{T}$ and $\mathbf{e}_{3}=(0, 0, 1)^{T}$, $\theta(\mathbf{e}_{i})=\{x_{i}\}$ for all $x_{i}\in U$.
We can find that $\mathbf{e}_{1} + \mathbf{e}_{2} \in \mathcal{N}_{\mathbf{R}}(A)$ and $\mathbf{e}_{2} + \mathbf{e}_{3} \in \mathcal{N}_{\mathbf{R}}(A)$.
Thus we know $(x_{1},x_{2})\in R_{\mathbf{R}}(A), (x_{2}, x_{3})\in R_{\mathbf{R}}(A), (x_{2},x_{1})\in R_{\mathbf{R}}(A)$ and $(x_{3}, x_{2})\in R_{\mathbf{R}}(A)$.
Therefore $R_{\mathbf{R}}(A)= \{(x_{1},x_{1}), (x_{2},x_{2}), (x_{3},x_{3}),$ $ (x_{1},x_{2}), (x_{2},x_{1}),$ $ (x_{2},x_{3}), (x_{3},x_{2})\}$.
It is clear that $R_{\mathbf{R}}(A)$ is not an equivalence relation on $U$ because $(x_{1},x_{2})\in R_{\mathbf{R}}(A)$ and $(x_{2},x_{3})\in R_{\mathbf{R}}(A)$ but $(x_{1},x_{3})\notin R_{\mathbf{R}}(A)$.
Similarly, if $F$ is binary field, then $R_{GF(2)}(A)=\{(x_{1},x_{1}), (x_{2},x_{2}), (x_{3},x_{3}), (x_{1},x_{3}), (x_{3},x_{1})\}$.
It is clear that $R_{GF(2)}(A)$ is an equivalence relation on $U$.
\end{example}

Form above example, we find that, over a field, the relation defined in Definition \ref{relationinducedbymatroid} may not be an equivalence relation.
However, it inspires us to consider whether the relation is an equivalence relation over binary field or not.
For convenience, we take $R(A)$ instead of $R_{GF(2)}(A)$ in the following section.

\begin{proposition}
\label{equivalencerelationoverbinaryfield}
Let $A=(a_{ij})_{m \times n}=[\mathbf{\alpha}_{1}, \mathbf{\alpha}_{2}, \cdots,$ $ \mathbf{\alpha}_{n}]$ be a matrix over binary field and $U$ the collection of the column labels of $A$.
$R(A)$ is an equivalence relation on $U$.
\end{proposition}

\begin{proof}
The reflexivity and symmetry of $R(A)$ are obvious.
Now we prove the transitivity of $R(M)$.
$\forall x_{i}, x_{j}, x_{k} \in U$, there exist identity vectors $\mathbf{e}_{i}, \mathbf{e}_{j}, \mathbf{e}_{k} \in V(n, GF(2))$ such that
$\theta(\mathbf{e}_{i})=\{x_{i}\}$, $\theta(\mathbf{e}_{j})=\{x_{j}\}$ and $\theta(\mathbf{e}_{k})=\{x_{k}\}$.
$x_{i}Rx_{j}$ and $x_{j}Rx_{k}$, if $x_{i}=x_{j}$ and $x_{j}=x_{k}$, then we obtain the result.
If $x_{i}=x_{j}$ and $A(\mathbf{e}_{j} + \mathbf{e}_{k})=0$, then $\mathbf{e}_{i}=\mathbf{e}_{j}$ and $A(\mathbf{e}_{j} + \mathbf{e}_{k})=0$, thus we have $A(\mathbf{e}_{i}+\mathbf{e}_{k})=0$, hence we obtain the result.
If $A(\mathbf{e}_{i} + \mathbf{e}_{j})=0$ and $A(\mathbf{e}_{j} + \mathbf{e}_{k})=0$, then
$\mathbf{0}=\mathbf{0}-\mathbf{0}=A(\mathbf{e}_{i} + \mathbf{e}_{j})-
A(\mathbf{e}_{j} + \mathbf{e}_{k})=A\mathbf{e}_{i} + A\mathbf{e}_{j}-A\mathbf{e}_{j} - A\mathbf{e}_{k}=A\mathbf{e}_{i} - A\mathbf{e}_{k}=
A\mathbf{e}_{i} + A(-\mathbf{e}_{k})=A\mathbf{e}_{i} + A\mathbf{e}_{k}= A(\mathbf{e}_{i} + \mathbf{e}_{k})$,
then $\mathbf{e}_{i} + \mathbf{e}_{k} \in \mathcal{N}_{GF(2)}(A)$, thus we obtain the result that if $x_{i} R x_{j}$ and $x_{j} R x_{k}$, then $x_{i} R x_{j}$.
Therefore, $R(A)$ is an equivalence relation on $U$.
\end{proof}

As we know, over a field, any two matrices which has the same number of columns may generate the same vector matroid.
However, the following proposition indicates that the equivalence relation induced by these two matrices are the same over binary field.

\begin{proposition}
\label{therelationbetweentwoequivalencerelationinducedbytworepresentablematrix}
Let $A_{1}$ and $A_{2}$ be two matrices which have not zero columns over binary field.
If $M_{GF(2)}[A_{1}]=M_{GF(2)}[A_{2}]$, then $R(A_{1})=R(A_{2})$.
\end{proposition}

\begin{proof}
Suppose the columns of $A_{1}=[\mathbf{\alpha}_{1}, \mathbf{\alpha}_{2}, \cdots,$ $ \mathbf{\alpha}_{n}]$ are, in order, labeled by $x_{1}, x_{2},$ $\cdots, x_{n}$, so does $A_{2}=[\mathbf{\alpha}^{'}_{1}, \mathbf{\alpha}^{'}_{2}, \cdots, \mathbf{\alpha}^{'}_{n}]$.
We know $U=\{x_{1}, x_{2}, \cdots, x_{n}\}$.
First, we prove $R(A_{1}) \subseteq R(A_{2})$.
For all $(x_{i},x_{j})\in R(A_{1})$, then $x_{i}=x_{j}$ or $\mathbf{e}_{i}+\mathbf{e}_{j}\in \mathcal{N}_{GF(2)}(A_{1})$, where
$\mathbf{e}_{i}, \mathbf{e}_{j} \in V(n, GF(2))$ satify $\theta(\mathbf{e}_{i})=\{x_{i}\}$ and $\theta(\mathbf{e}_{j})= \{x_{j}\}$.
If $x_{i} = x_{j}$, then we obtain the result.
If $\mathbf{e}_{i}+\mathbf{e}_{j}\in \mathcal{N}_{GF(2)}(A_{1})$, then
$\theta(\mathbf{e}_{i}+\mathbf{e}_{j})=\{x_{i}, x_{j}\} \in \mathcal{D}(M_{GF(2)}[A_{1}])=\mathcal{D}(M_{GF(2)}[A_{2}])$ according to Proposition \ref{P:nullspaceanddependentset}.
Thus the columns of $A_{2}$ labeled by $x_{i}$ and $x_{j}$, respectively, are linearly dependent in $V(n, GF(2))$.
Since $A_{2}$ has not zero columns, $\mathbf{\alpha}^{'}_{i} + \mathbf{\alpha}^{'}_{j} = \mathbf{0}$, that is, $A_{2}(\mathbf{e}_{i}+\mathbf{e}_{j}) = \mathbf{0}$.
Thus $\mathbf{e}_{i}+\mathbf{e}_{j} \in \mathcal{N}_{GF(2)}(A_{2})$ which implies $(x_{i}, x_{j})\in R(A_{2})$.
Hence $R(A_{1})\subseteq R(A_{2})$.
Similarly, we can prove $R(A_{2})\subseteq R(A_{1})$.
\end{proof}

Over binary field, for an equivalence relation, one can obtain a matrix through Theorem \ref{T: circuitsbetweenmatroidandvectormatroid}, and then obtain the other equivalence relation through Definition \ref{relationinducedbymatroid}.
What about the relationship between these two equivalence relations?

\begin{proposition}
\label{relationbetweenRandA}
Let $R$ be an equivalence relation on $U$.
If $M(R)=M_{GF(2)}[A(R)]$, then $R(A(R))=R$.
\end{proposition}

\begin{proof}
We may as well suppose the columns of $A(R)$ are, in order, labeled by $x_{1}, x_{2}, $ $\cdots, x_{n}$.
Thus $U=\{x_{1}, x_{2}, \cdots, x_{n}\}$.
Since for all $i\in \{1, 2, \cdots, n\}$, $\{x_{i}\}$ is an independent set of $M(R)$.
Thus $A(R)$ dose not contain zero columns.
For all $(x_{i}, x_{j})\in R(A(R))$, then $x_{i}=x_{j}$ or $\mathbf{e}_{i}+\mathbf{e}_{j}\in \mathcal{N}_{GF(2)}(A(R))$,
where $\mathbf{e}_{i}, \mathbf{e}_{j} \in V(n, GF(2))$ satisfy $\theta(\mathbf{e}_{i})=\{x_{i}\}$ and $\theta(\mathbf{e}_{j})=\{x_{j}\}$.
If $x_{i}=x_{j}$, then $(x_{i}, x_{j})\in R$ for $R$ is an equivalence relation.
If $x_{i} \neq x_{j}$, then $\mathbf{e}_{i}+\mathbf{e}_{j} \in \mathcal{N}_{GF(2)}(A(R))$.
According to Proposition \ref{P:nullspaceanddependentset}, then $\theta(\mathbf{e}_{i}+\mathbf{e}_{j})=\{x_{i},x_{j}\}\in \mathcal{D}(M(R))$.
But $\{x_{i}\}$ or $\{x_{j}\}$ is an independent set of $M(R)$.
Thus $\{x_{i}, x_{j}\} \in \mathcal{C}(M(R))$, that is, there exists $P\in U/R$ such that $\{x_{i}, x_{j}\} \in P$ which implies
$(x_{i}, x_{j})\in R$.
Conversely, $\forall (x_{i}, x_{j})\in R$, then there exists $P\in U/R$ such that $\{x_{i}, x_{j}\} \in P$, that is,
$\{x_{i},x_{j}\}\in \mathcal{C}(M(R))$.
According to Theorem \ref{circuitandresolutionspace}, then $\mathbf{e}_{i} + \mathbf{e}_{j}\in \mathcal{N}_{GF(2)}(A(R))$, that is,
$(x_{i}, x_{j}) \in R(A(R))$.
Thus $R(A(R))=R$.
\end{proof}

Since $B(R)$ is a matrix which can generate matroid $M(R)$ over binary field, we can obtain the following corollary.

\begin{corollary}
Let $R$ be an equivalence relation on $U$ and $B(R)$ a matrix representation of $R$.
$R(B(R))=R$.
\end{corollary}

Next, we define a special type of matrix.
Over binary field, the null space of this type of matrix has close relation with equivalence relation.

\begin{definition}
\label{specialmatrix}
Let $F$ be a field and $A=(a_{ij})_{m \times n}=[\mathbf{\alpha}_{1}, \mathbf{\alpha}_{2}, \cdots, \mathbf{\alpha}_{n}]$ a matrix over $F$.
If $A$ satisfies the following conditions:\\
(1) for all $i \in \{1, 2, \cdots, n\}$, $\mathbf{\alpha}_{i} \neq 0$,\\
(2) for all $k \in \{2, \cdots, n\}$, if $r_{F}[\mathbf{\alpha}_{i_{1}}, \mathbf{\alpha}_{i_{2}}, \cdots, \mathbf{\alpha}_{i_{k}}]<k$, then there exists $ \{\mathbf{\alpha}_{i_{p}}, \mathbf{\alpha}_{i_{q}}\} \subseteq
\{\mathbf{\alpha}_{i_{1}}, \mathbf{\alpha}_{i_{2}}, \cdots, \mathbf{\alpha}_{i_{k}}\}$ such that $r_{F}[\mathbf{\alpha}_{i_{p}}, \mathbf{\alpha}_{i_{q}}]< 2$,\\
then $A$ is called a binary dependence matrix and we denote the set of this type of matrices as $\mathcal{A}$.
\end{definition}

The following proposition shows the relation between the matroid $M(R)$ induced by an equivalence relation and the collection of binary dependence matrices $\mathcal{A}$.

\begin{proposition}
\label{A(R)isabinarydependentmatrix}
Let $R$ be an equivalence relation on $U$.
If $M(R)=M_{F}[A(R)]$, then $A(R)\in \mathcal{A}$.
\end{proposition}

\begin{proof}
Suppose $U=\{x_{1}, x_{2}, \cdots, x_{n}\}$, $U/R=\{P_{1}, P_{2}, \cdots, P_{s}\}$ and the columns of $A(R)=[\mathbf{\alpha}_{1}, \mathbf{\alpha}_{2}, \cdots, \mathbf{\alpha}_{n}]$ are labeled, in order, by $x_{1}, x_{2}, \cdots, x_{n}$.
According to Proposition \ref{matroidinducedbyenquivalencerelation}, we know that $\mathbf{\alpha}_{i} \neq \mathbf{0}$ for all $i \in \{1, 2, \cdots, n \}$; otherwise, $M(R)$ has single-point sets as its circuits which implies contradictory.
Since $M(R)=M_{F}[A(R)]$, for all $k \geq 2$, if $\mathbf{\beta}_{i_{1}}, \mathbf{\beta}_{i_{2}}, \cdots, $ $\mathbf{\beta}_{i_{k}}$ are linearly dependent over $F$,
then $\{x_{i_{1}}, x_{i_{2}}, \cdots, x_{i_{k}}\} \in D(M(R))$.
Thus there exists $C \in C(M(R))$ such that $C \subseteq \{x_{i_{1}}, x_{i_{2}}, \cdots, x_{i_{k}}\}$.
According to Proposition \ref{matroidinducedbyenquivalencerelation}, we may as well suppose $C=\{x_{i_{p}}, x_{i_{q}}\}$.
Then the columns labeled by the elements of $C$ are linearly dependent over $F$, i.e., $\mathbf{\beta}_{i_{p}},\mathbf{\beta}_{i_{q}}$
are linearly dependent over $F$.
Hence $A(R)\in \mathcal{A}$.
\end{proof}

Over a field, matrix $B(R)$ can induce a matroid and the matroid is $M(R)$.
Based on the above proposition, we can obtain the following result.

\begin{corollary}
Let $R$ be an equivalence relation on $U$ and $B(R)$ a matrix representation of $R$. $B(R) \in \mathcal{A}$.
\end{corollary}

If we first convert a binary dependence matrix into an equivalence relation, then covert the equivalence relation into a matrix.
The matroid induced by the second conversion over binary field is the one induced by the first conversion.

\begin{proposition}
\label{relationbetweenA(R(A))andA}
Let $A \in \mathcal{A}$.
$M_{GF(2)}[A(R(A))] = M_{GF(2)}[A]$.
\end{proposition}

\begin{proof}
Suppose the columns of $A$ are labeled, in order, by $x_{1}, x_{2}, \cdots, x_{n}$, so dose $A(R(A))$.
Then $U=\{x_{1}, x_{2}, \cdots, x_{n}\}$.
For all $C \in \mathcal{C}(M_{GF(2)}[A(R(A))])$, according to Proposition~\ref{matroidinducedbyenquivalencerelation}, we may as well suppose $C=\{x_{i}, x_{j}\}$ and there exists $P_{i} \in U/R(A)$ such that $\{x_{i}, x_{j}\}\in P_{i}$, that is, $(x_{i}, x_{j}) \in R(A)$.
Then for $\mathbf{e}_{i}, \mathbf{e}_{j}\in V(n,GF(2))$ satisfying $\theta(\mathbf{e}_{i})=\{x_{i}\}$ and $\theta(\mathbf{e}_{j})=\{x_{j}\}$, we have $A(\mathbf{e}_{i} + \mathbf{e}_{j})=0$, that is, the columns of $A$ labeled by $x_{i}$ and $x_{j}$ are linearly dependent in $V(n, GF(2))$.
Since $A \in \mathcal{A}$, $A$ dose not contain zero columns, thus the column of $A$ labeled by $x_{i}$ or $x_{j}$ is linearly independent in
$V(n, GF(2))$.
Hence $C\in \mathcal{C}(M_{GF(2)}[A])$, that is, $\mathcal{C}(M_{GF(2)}[A(R(A))]) \subseteq \mathcal{C}(M_{GF(2)}[A])$.
Conversely, $\forall C\in \mathcal{C}(M_{GF(2)}[A])$, the columns of $A$ labeled by the elements of $C$ are linearly dependent in $V(n, GF(2))$.
Since $A \in \mathcal{A}$, then there exists $C_{1}$ which has only two elements such that $C_{1} \subseteq C$.
We may as well suppose $C_{1}=\{x_{i}, x_{j}\}$.
Since $A \in \mathcal{A}$, the column of $A$ labeled by $x_{i}$ or $x_{j}$ is linearly independent in
$V(n, GF(2))$.
Thus $C_{1} \in \mathcal{C}(M_{GF(2)}[A])$.
Based on circuit axiom and Theorem~\ref{circuitandresolutionspace}, we can obtain $C=C_{1}=\{x_{1}, x_{j}\} \in Min\{\theta(\mathcal{N}_{GF(2)}(A))-\emptyset\}$, that is,
$\mathbf{e}_{i} + \mathbf{e}_{j} \in \mathcal{N}_{GF(2)}(A)$.
Hence, $(x_{i}, x_{j}) \in R(A)$, that is, there exists $P_{i}\in U/R(A)$ such that $\{x_{i}, x_{j}\} \subseteq P_{i}$, thus $C=\{x_{i}, x_{j}\} \in \mathcal{C}(M_{GF(2)}[A(R(A))])$.
Therefore, we obtain $\mathcal{C}(M_{GF(2)}[A]) \subseteq \mathcal{C}(M_{GF(2)}[A(R(A))])$, that is, $M_{GF(2)}[A(R(A))] = M_{GF(2)}[A]$.
\end{proof}

The following result is the combination of Theorem~\ref{circuitandresolutionspace} and Proposition~\ref{relationbetweenA(R(A))andA}.

\begin{proposition}
\label{therelationbetweenA(R(A)andA}
Let $A \in \mathcal{A}$.
 $Min (\theta(\mathcal{N}_{GF(2)}(A(R(A))))-\emptyset) = Min (\theta(\mathcal{N}_{GF(2)}(A))-\emptyset)$.
\end{proposition}

Suppose $NS = \{Min (\theta(\mathcal{N}_{GF(2)}(A))-\emptyset): A \in \mathcal{A}\}$ and $\mathbb{R} = \{R: R$ is an equivalence relation on $U\}$.
Proposition \ref{relationbetweenRandA}, \ref{A(R)isabinarydependentmatrix} and \ref{therelationbetweenA(R(A)andA} indicate that there is a one-to-one correspondence between $\mathbb{R}$ and $NS$.
The following theorem shows a deeper relation between them.
In fact, $(\mathbb{R}, \bigcap)$ and $(NS, \bigcap)$ are algebra isomorphism.

\begin{theorem}
$(\mathbb{R}, \bigcap) \cong (NS, \bigcap)$.
\end{theorem}

\begin{proof}
For any equivalence relation, we can obtain a matrix $A(R)$ through Theorem~\ref{T: circuitsbetweenmatroidandvectormatroid}.
Based on Proposition \ref{A(R)isabinarydependentmatrix}, we know $A(R)$ is a binary dependence matrix, thus we define an operator $f:\mathbb{R}\rightarrow NS$ as follows:
$f(R)=Min(\theta(\mathcal{N}_{GF(2)}(A(R))-\emptyset)$.
First we need to prove $f$ is bijection.
According to Theorem \ref{circuitandresolutionspace}, we know $Min (\theta(\mathcal{N}_{GF(2)}(\mathcal{A}))-\emptyset) = C(R)$.
According to Proposition \ref{matroidinducedbyenquivalencerelation}, we know $f$ is injection.
$\forall Min(\theta(\mathcal{N}_{GF(2)}(A)-\emptyset) \in NS$, let $R=R(A)$.
According to Proposition \ref{therelationbetweenA(R(A)andA}, $f(R(A))=Min(\theta(\mathcal{N}_{GF(2)}(A)$ $-\emptyset)$, that is, $f$ is a surjection.
In the following, we need to prove $f(R_{1} \bigcap R_{2}) = f(R_{1}) \bigcap f(R_{2})$, that is, $Min(\theta(\mathcal{N}_{GF(2)}(A(R_{1} \bigcap R_{2}))-\emptyset) = Min (\theta$ $(\mathcal{N}_{GF(2)}(A(R_{1}))$ $-\emptyset) \bigcap Min(\theta(\mathcal{N}_{GF(2)}(A(R_{2}))-\emptyset)$.
For all $\{x_{i}, x_{j}\} \in Min(\theta(\mathcal{N}_{GF(2)}(A(R_{1} \bigcap R_{2}))$ $-\emptyset)$, then $\{x_{i}, x_{j}\} \in \mathcal{C}(R_{1}\bigcap R_{2})$, i.e., there exists $P_{k} \in U/(R_{1}\bigcap R_{2})$ such that $\{x_{i}, x_{j}\}\in P_{k}$.
Thus there exist $P^{1}_{s}\in U/R_{1}$ and $P^{2}_{t}\in U/R_{2}$ such that $\{x_{i}, x_{j}\}\in P^{1}_{s}$ and $\{x_{i}, x_{j}\}\in P^{2}_{t}$.
Hence, $\{x_{i}, x_{j}\}\in \mathcal{C}(R_{1}) \bigcap \mathcal{C}(R_{2})= Min(\theta(\mathcal{N}_{GF(2)}(A(R_{1}))-\emptyset) \bigcap Min(\theta(\mathcal{N}_{GF(2)}(A(R_{2})-\emptyset)$.
Therefore, $Min(\theta(\mathcal{N}_{GF(2)}(A(R_{1} \bigcap R_{2}))-\emptyset) \subseteq  Min(\theta(\mathcal{N}_{GF(2)}(A(R_{1}))-\emptyset) \bigcap Min(\theta(\mathcal{N}_{GF(2)}$ $(A(R_{2}))-\emptyset)$.
For all $\{x_{i}, x_{j}\} \in Min(\theta(\mathcal{N}_{GF(2)}(A(R_{1}))-\emptyset) \bigcap Min(\theta(\mathcal{N}_{GF(2)}(A(R_{2})-\emptyset)$, that is, $\{x_{i}, x_{j}\} \in \mathcal{C}(R_{1})$ $ \bigcap \mathcal{C}(R_{2})$,
then there exist $P^{1}_{s}\in U/R_{1}$ and $P^{2}_{t}\in U/R_{2}$ such that $\{x_{i}, x_{j}\}\in P^{1}_{s}$ and $\{x_{i}, x_{j}\}\in P^{2}_{t}$, i.e.,
$(x_{i}, x_{j})\in R_{1}\bigcap R_{2}$.
Hence there exists $P_{k}\in U/(R_{1}\bigcap R_{2})$ such that $\{x_{i}, x_{j}\}\in P_{k}$, that is, $\{x_{i}, x_{j}\}\in \mathcal{C}(R_{1}\bigcap R_{2})$.
Therefore, $Min(\theta(\mathcal{N}_{GF(2)}(A(R_{1}))$ $-\emptyset) \bigcap Min(\theta(\mathcal{N}_{GF(2)}(A(R_{2})-\emptyset) \subseteq Min(\theta(\mathcal{N}_{GF(2)}(A(R_{1} \bigcap R_{2}))-\emptyset) $.
In a word, $(\mathbb{R}, \bigcap) \cong (Min (\theta(\mathcal{N}_{GF(2)}(\mathcal{A}))$ $-\emptyset), \bigcap)$.
\end{proof}

Isomorphisms are studied in mathematics in order to extend insights from one phenomenon to others.
If two algebra systems are isomorphic, then they can be regarded as similarity.
Hence, the study of the rough sets is equal to the study of the family of sets which any member is a collection of the minimal non-empty sets that are supports of members of null space of a binary dependence matrix.

\section{Conclusions}
\label{S:Conclusions}
In this paper, we employed the matrix approach to study rough sets over a field.
A matroidal structure of rough sets was constructed through matrix, and the matrix approaches such as null space were employed to study the characteristics of the matroidal structure.
We also found that a family of equivalence relations and a family of sets, which any member is a collection of the minimal non-empty sets that are supports of members of null space of a binary dependence matrix, are algebra isomorphism.
Though some works have been studied in this paper, there are also many interesting topics deserving further investigation.
In the future, we will study rough sets from the following two aspects.
On one hand, nullity which an important concept in matroid theory will be introduced to study rough sets.
On the other hand, matrix will be promoted to study covering-based rough sets and relation-based rough sets.

\section{Acknowledgments}
This work is supported in part by the National Natural Science Foundation of China under Grant No. 61170128,
the Natural Science Foundation of Fujian Province, China, under Grant Nos. 2011J01374 and 2012J01294,
and the Science and Technology Key Project of Fujian Province, China, under Grant No. 2012H0043.



\end{document}